\documentclass{article}

\pdfoutput=1

\usepackage{graphicx}
\usepackage{subfigure}
\usepackage{natbib}
\usepackage{algorithm}
\usepackage{algorithmic}
\usepackage{dblfloatfix}
\usepackage[accepted]{icml2014}

\usepackage{url}

\usepackage{todonotes}

\usepackage{amsmath}
\usepackage{amssymb}
\usepackage{amsfonts}
\usepackage{amsthm}

\newtheorem{definition}{Definition}
\newtheorem{proposition}[definition]{Proposition}
\newtheorem{theorem}[definition]{Theorem}

\newtheorem{lemma}[definition]{Lemma}

\newcommand{\diam}{\operatorname{diam}}
\newcommand{\E}[1]{\mathbb{E}\left[#1\right]}
\newcommand{\EE}[2]{\mathbb{E}_{#1}\left[#2\right]}
\newcommand{\R}{\mathbb{R}}
\newcommand{\D}{\mathcal{D}}
\newcommand{\1}[1]{\mathbb{I}\left\{#1\right\}}
\renewcommand{\exp}[1]{\operatorname{exp}\left(#1\right)}
\renewcommand{\P}[1]{\mathbb{P}\left(#1\right)}

\icmltitlerunning{Random Forests In Theory and In Practice}

\begin{document}
\twocolumn[
\icmltitle{Narrowing the Gap: Random Forests\\In Theory and In Practice}

% It is OKAY to include author information, even for blind
% submissions: the style file will automatically remove it for you
% unless you've provided the [accepted] option to the icml2013
% package.
\icmlauthor{Misha Denil}{mdenil@cs.ubc.ca}
\icmlauthor{David Matheson}{davidm@cs.ubc.ca}
\icmladdress{University of British Columbia, Canada}
\icmlauthor{Nando de Freitas}{nando@cs.ox.ac.uk}
\icmladdress{Oxford University, United Kingdom}
% \icmlauthor{Your CoAuthor's Name}{email@coauthordomain.edu}
% \icmladdress{Their Fantastic Institute,
%             27182 Exp St., Toronto, ON M6H 2T1 CANADA}

% You may provide any keywords that you
% find helpful for describing your paper; these are used to populate
% the "keywords" metadata in the PDF but will not be shown in the document
\icmlkeywords{random forests, ensembles, consistency, kinect}

\vskip 0.3in
]

\begin{abstract}
  Despite widespread interest and practical use, the theoretical properties of
  random forests are still not well understood.  In this paper we contribute to
  this understanding in two ways.  We present a new theoretically tractable
  variant of random regression forests and prove that our algorithm is
  consistent.  We also provide an empirical evaluation, comparing our algorithm
  and other theoretically tractable random forest models to the random forest
  algorithm used in practice.  Our experiments provide insight into the relative
  importance of different simplifications that theoreticians have made to obtain
  tractable models for analysis.
\end{abstract}

\section{Introduction}

Random forests are a type of ensemble method which makes predictions by
averaging over the predictions of several independent base models.  Since
its introduction by \citet{Breiman2001:random_forests} the random forests
framework has been extremely successful as a general purpose classification and
regression method.

Despite their widespread use, a gap remains between the theoretical
understanding of random forests and their practical use.  A variety of random
forest algorithms have appeared in the literature, with great practical success.
However, these algorithms are difficult to analyze, and the basic mathematical
properties of even the original variant are still not well
understood~\cite{Biau2012}.

This state of affairs has led to a polarization between theoretical and
empirical contributions to the literature.  Empirically focused papers describe
elaborate extensions to the basic random forest framework, adding domain
specific refinements which push the state of the art in performance, but come
with no guarantees
\cite{Schroff2008,Shotton2011,Montillo2011,Xiong2012,Zik2012}.
In contrast, theoretical papers focus on simplifications of the standard
framework where analysis is more tractable.  Notable contributions in this
direction are the recent papers of \citet{biau08} and \citet{Biau2012}.

In this paper we present a new variant of random regression forests with
tractable theory, which relaxes two of the key simplifying assumptions from
previous works.  We also provide an empirical comparison between standard random
forests and several models which have been analyzed by the theory community.
Our algorithm achieves the closest match between theoretically tractable models
and practical algorithms to date, both in terms of similarity of the algorithms
and in empirical performance.

Empirical comparison of the theoretical models, something which has not
previously appeared in the literature, provides important insight into the
relative importance of the different simplifications made to the standard
algorithm to enable tractable analysis.

%%% Local Variables:
%%% mode: latex
%%% TeX-master: "regression-forests-2013"
%%% End:

\section{Related work}

Random forests~\cite{Breiman2001:random_forests} were originally conceived as a
method of combining several CART~\cite{breiman84} style decision trees using
bagging~\cite{breiman96}.  Their early development was influenced by the random
subspace method of \citet{ho98}, the approach of random split selection from
\citet{dietterich2000} and the work of \citet{amit1997} on feature selection.
Several of the core ideas used in random forests were also present in the early
work of \citet{kwokt88} on ensembles of decision trees.

In the years since their introduction, random forests have grown from a single
algorithm to an entire framework of models~\cite{criminisi2011decision}, and
have been applied to great effect in a wide variety of fields~\cite{Svetnik2003,
  Prasad2006, Cutler2007, Shotton2011, criminisi2013}.

In spite of the extensive use of random forests in practice, the mathematical
forces underlying their success are not well understood.  The early theoretical
work of \citet{Breiman2004} for example, is essentially based on intuition and
mathematical heuristics, and was not formalized rigorously until quite
recently \cite{Biau2012}.

There are two main properties of theoretical interest associated with random
forests.  The first is consistency of estimators produced by the algorithm which
asks (roughly) if we can guarantee convergence to an optimal estimator as the
data set grows infinitely large.  Beyond consistency we are also interested in
rates of convergence.  In this paper we focus on consistency, which,
surprisingly, has not yet been established even for Breiman's original
algorithm.

Theoretical papers typically focus on stylized versions of the algorithms used
in practice.  An extreme example of this is the work of \citet{genuer2010risk,
  genuer2012}, which studies a model of random forests in one dimension with
completely random splitting.  In exchange for simplification researchers
acquire tractability, and the tact assumption is that theorems proved for
simplified models provide insight into the properties of their more
sophisticated counterparts, even if the formal connections have not been
established.

An important milestone in the development of the theory of random forests is the
work of \citet{biau08}, which proves the consistency of several randomized
ensemble classifiers.  Two models studied in \citet{biau08} are direct
simplifications of the algorithm from \citet{Breiman2001:random_forests}, and
two are simple randomized neighborhood averaging rules, which can be viewed as
simplifications of random forests from the perspective of \citet{Lin2002}.

More recently \citet{Biau2012} has analyzed a variant of random forests
originally introduced in \citet{Breiman2004} which is quite similar to the
original algorithm.  The main differences between the model in \citet{Biau2012}
and that of \citet{Breiman2001:random_forests} are in how candidate split points
are selected and that the former requires a second independent data set to fit
the leaf predictors.

While the problem of consistency of Breiman's algorithm remains open, some
special cases have proved tractable.  In particular, \citet{Meinshausen2006} has
shown that a model of random forests for quantile regression is consistent and
\citet{Ishwaran10} have shown the consistency of their survival forests model.
\citet{denil2013consistency} have shown the consistency of an online version of
random forests.

%%% Local Variables:
%%% mode: latex
%%% TeX-master: "regression-forests-2013"
%%% End:

\section{Random Forests}
\label{sec:random-forests}

In this section we briefly review the random forests framework.  For a more
comprehensive review we refer the reader to \citet{Breiman2001:random_forests}
and \citet{criminisi2011decision}.

Random forests are built by combining the predictions of several trees, each of
which is trained in isolation.  Unlike in boosting \cite{schapire12} where the
base models are trained and combined using a sophisticated weighting scheme,
typically the trees are trained independently and the predictions of the trees
are combined through averaging.

There are three main choices to be made when constructing a random tree.  These
are (1) the method for splitting the leafs, (2) the type of predictor to use in
each leaf, and (3) the method for injecting randomness into the trees.

Specifying a method for splitting leafs requires selecting the shapes of
candidate splits as well as a method for evaluating the quality of each
candidate.  Typical choices here are to use axis aligned splits, where data are
routed to sub-trees depending on whether or not they exceed a threshold value in
a chosen dimension; or linear splits, where a linear combination of features are
thresholded to make a decision.  The threshold value in either case can be
chosen randomly or by optimizing a function of the data in the leafs.

In order to split a leaf, a collection of candidate splits are generated and a
criterion is evaluated to choose between them.  A simple strategy is to choose
among the candidates uniformly at random, as in the models analyzed in
\citet{biau08}.  A more common approach is to choose the candidate split which
optimizes a purity function over the leafs that would be created.  Typical
choices here are to maximize the information gain, or the Gini gain
\cite{Hastie2013}.

% The most common choice for predictors in each leaf is to use the majority vote
% over the training points which fall in that leaf.  \citet{criminisi2011decision}
% explore the use of several different leaf predictors for regression and manifold
% learning, but these generalizations are beyond the scope of this paper.  We
% consider majority vote classifiers in our model.

The most common choice for predictors in each leaf is to use the average
response over the training points which fall in that leaf.
\citet{criminisi2011decision} explore the use of several different leaf
predictors for regression and other tasks, but these generalizations are beyond
the scope of this paper.  We consider only simple averaging predictors here.

Injecting randomness into the tree construction can happen in many
ways.  The choice of which dimensions to use as split candidates at
each leaf can be randomized, as well as the choice of coefficients for
random combinations of features.  In either case, thresholds can be
chosen either randomly or by optimization over some or all of the data
in the leaf.

Another common method for introducing randomness is to build each tree using a
bootstrapped or sub-sampled data set.  In this way, each tree in the forest is
trained on slightly different data, which introduces differences between the
trees.

%%% Local Variables: 
%%% mode: latex
%%% TeX-master: "regression-forests-2013"
%%% End: 

\section{Algorithm}
\label{sec:algorithm}

In this section we describe the workings of our random forest algorithm.  Each
tree in the random regression forest is constructed independently.  Unlike
the random forests of \citet{Breiman2001:random_forests} we do not preform
bootstrapping between the different trees.

\subsection{Tree construction}

Each node of the tree corresponds to a rectangular subset of $\R^D$, and at each
step of the construction the cells associated with leafs of the tree form a
partition of $\R^D$.  The root of the tree corresponds to all of $\R^D$.  At
each step of the construction a leaf of the tree is selected for expansion.

In each tree we partition the data set randomly into two parts, each of which
plays a different role in the tree construction.  We refer to points assigned to
the different parts as \emph{structure} and \emph{estimation} points
respectively.

\textbf{Structure points} are allowed to influence the shape of the tree.  They
are used to determine split dimensions and split points in each internal node of
the tree.  However, structure points are not permitted to effect the predictions
made in the tree leafs.

\textbf{Estimation points} play the dual role.  These points are used to fit the
estimators in each leaf of the tree, but have no effect on the shape of the tree
partition.

The data are randomly partitioned in each tree by assigning each point to the
structure or estimation part with equal probability.  This partition is required
to ensure consistency; however, there is no reason we cannot have additional
parts.  For instance, we could assign some points to a third, ignored part of
the partition in order to fit each tree on a subset of the data.  However, we
found that subsampling generally hurts performance, so we do not pursue this
idea further.

\subsection{Leaf expansion}

\begin{figure}[t]
  \centering
  \includegraphics[width=1\linewidth]{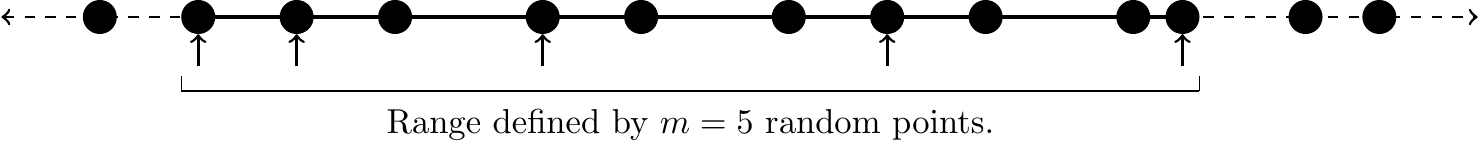}
  \caption{The search range in each candidate dimension is defined by choosing
    $m$ random structure points (indicated by arrows) and searching only over
    the range defined by those points.  Candidate split points can only be
    selected in the region denoted by the solid line; the dashed areas are not
    eligible for splitting.}
  \label{fig:split-points}
\end{figure}

When a leaf is selected for expansion we select, at random,
$\min(1+\operatorname{Poisson}(\lambda), D)$ distinct candidate dimensions.  We
choose a split point for the leaf by searching over the candidate split points
in each of the candidate dimensions.  

A key difference between our algorithm and standard random forests is how we
search for split points in candidate dimensions.  In a standard random forest
points are projected into the candidate dimension and every possible split point
is evaluated as a candidate split point.  In our algorithm we restrict the range
of the search by first selecting $m$ of the structure points in the leaf and
evaluating candidate split points only over the range defined by these points.
Restricting the range in this way forces the trees to be (approximately)
balanced, and is depicted in Figure~\ref{fig:split-points}.

For each candidate split point $S$ we compute the reduction in squared error,
\begin{align*}
  \operatorname{Err}(L) &= \frac{1}{N^s(L)}\sum_{\substack{Y_j\in L\\I_j=s}}
  (Y_j-\bar{Y}^L)^2
  \\
  I(S) &= \operatorname{Err}(A) - \operatorname{Err}(A') - \operatorname{Err}(A'')
\end{align*}
where $A$ is the leaf to be split, $A'$ and $A''$ are the two children which
would be created by splitting $A$ at $S$.  The notation $\bar{Y}^A$ denotes the
empirical mean of the structure points falling in the cell $A$ and $N^s(A)$
counts the number of structure points in $A$.  The variables $I_j\in\{e,s\}$ are
indicators which denote whether the point $(X_j,Y_j)$ is a structure or
estimation point.

The split point is chosen as the candidate which maximizes $I(S)$ without
creating any children with fewer than $k_n$ estimation points, where $n$ denotes
the size of the training set.  If no such candidate is found then expansion is
stopped.

\subsection{Prediction}

Once the forest has been trained it can be used to make predictions for new
unlabeled data points.  To make a prediction for a query point $x$, each tree
independently predicts
\begin{align*}
  f_n^j(x) = \frac{1}{N^e(A_n(x))} \sum_{\substack{Y_i \in A_n(x)\\I_i=e}} Y_i
\end{align*}
and the forest averages the predictions of each tree
\begin{align*}
  f_n^{(M)}(x) = \frac{1}{M} \sum_{j=1}^M f_n^j(x)
\end{align*}
Here $A_n(x)$ denotes the leaf containing $x$ and $N^e(A_n(x))$ denotes the
number of estimation points it contains.  Note that the predictions made by each
tree depend only on the estimation points in that tree; however, since points
are assigned to the structure and estimation parts independently in each tree,
structure points in one tree have the opportunity to contribute to the
prediction as estimation points in another tree.

%%% Local Variables: 
%%% mode: latex
%%% TeX-master: "regression-forests-2013"
%%% End: 

\section{Consistency}

In this section we prove consistency of the random regression forest model
described in this paper.  We denote a tree partition created by our algorithm
trained on data $\D_n = \{(X_i,Y_i)\}_{i=1}^n$ as $f_n$.  As $n$ varies we
obtain a sequence of classifiers and we are interested in showing that the
sequence $\{f_n\}$ is consistent as $n\to\infty$.  More precisely,
\begin{definition}
  A sequence of estimators $\{f_n\}$ is consistent for a certain distribution on
  $(X,Y)$ if the value of the risk functional
  \begin{align*}
    R(f_n) &= \EE{X,Z,\D_n}{|f_n(X, Z, \D_n) - f(X)|^2}
  \end{align*}
  converges to $0$ as $n\to\infty$, where $f(x) = \E{Y|X=x}$ is the (unknown)
  regression function.
\end{definition}

In order to show that our random forest classifier is consistent, we will take
advantage of its structure as an empirical averaging estimator.
\begin{definition}
  A (randomized) empirical averaging estimator is an estimator that averages a
  fixed number of (possibly dependent) base estimators, i.e.\
  \begin{align*}
    f_n^{(M)}(x,Z^{(M)},\D_n) = \frac{1}{M}\sum_{j=1}^M f_n(x, Z^j,\D_n)
  \end{align*}
  where $Z^{(M)} = (Z^1, \ldots, Z^M)$ is composed of $M$ (possibly dependent)
  realizations of $Z$.
\end{definition}

The first step of our construction is to show that the consistency of the random
regression forest is implied by the consistency of the trees it is composed of.
The following proposition makes this assertion precise.  A similar result was
shown by \citet{biau08} for binary classifiers and a corresponding mutli-class
generalization appears in \citet{denil2013consistency}.  For regression, it is
particularly straightforward.

\begin{proposition}
  Suppose $\{f_n\}$ is a sequence of consistent estimators.  Then
  $\{f_n^{(M)}\}$, the sequence of empirical averaging estimators obtained by
  averaging $M$ copies of $\{f_n\}$ with different randomizing variables is also
  consistent.
  \label{prop:consistent-base}
\end{proposition}
\begin{proof}
  We must show that $R(f_n^{(M)}) \to 0$.  Compute
  \begin{align*}
    &R(f_n^{(M)}) = 
    \\&\hspace{0.5cm}
    \EE{X,Z^{(M)},\D_n}{| \frac{1}{M} \sum_{j=1}^Mf_n(x,
      Z^j,\D_n) - f(x)|^2}
  \end{align*}
  by the triangle inequality and the fact that $(\sum_{i=1}^na_i)^2 \le
  n\sum_{i=1}^na_i^2$,
  \begin{align*}
    &\le \frac{1}{M} \sum_{j=1}^M \EE{X,Z^{j},\D_n}{| f_n(x, Z^j,\D_n) -
      f(x)|^2}
    \\
    &= R(f_n) \to 0
  \end{align*}
  which is the desired result.
\end{proof}

Proposition~\ref{prop:consistent-base} allows us to focus our attention on the
consistency of each of the trees in the regression forest.  The task of proving
the tree estimators are consistent is greatly simplified if we condition on the
partition of the data into structure and estimation points.  Conditioned on the
partition, the shape of the tree becomes independent of the estimators in the
leafs.  The following proposition shows that, under certain conditions, proving
consistency conditioned on the partitioning variables is sufficient.

\begin{proposition}
  Suppose $\{f_n\}$ is a sequence of estimators which are conditionally
  consistent for some distribution on $(X,Y)$ based on the value of some
  auxiliary variable $I$.  That is,
  \begin{align*}
    \lim_{n\to\infty}\EE{X,Z,\D_n}{|f_n(X,Z,I,\D_n)-f(x)|^2 \,|\, I} = 0
  \end{align*}
  for all $I\in\mathcal{I}$ and that $\nu$ is a distribution on $I$.  Moreover,
  suppose $f(x)$ is bounded.  If these conditions hold and if $\nu(\mathcal{I})
  = 1$ and each $f_n$ is bounded with probability 1, then $\{f_n\}$ is
  unconditionally consistent, i.e.\ $R(f_n) \to 0$.
  \label{prop:condition-on-full-measure}
\end{proposition}
\begin{proof}
  Note that
  \begin{align*}
    R(m_n) &= \EE{X,Z,I,\D_n}{|f_n(X, Z, I, \D_n) - f(X)|^2}
    \\
    &= \EE{I}{\EE{X,Z,\D_n}{|f_n(X, Z, I, \D_n) - f(X)|^2 \,|\,I}}
  \end{align*}
  and
  \begin{align*}
    &\EE{X, Z, I, \D_n}{|f_n(X, Z, I, \D_n) - f(X)|^2}
    \\
    &\le \EE{X,Z,I,\D_n}{|f_n(X, Z, I, \D_n)|^2} + \EE{X}{|f(X)|^2}
    \\
    &\le \sup_x \EE{Z,I,\D_n}{|f_n(x, Z, I, \D_n)|^2} + \sup_x |f(x)|^2
  \end{align*}
  Both of these terms are finite by the boundedness assumptions.  This means we
  can apply the dominated convergence theorem to obtain
  \begin{align*}
    &\lim_{n\to\infty} R(f_n) = \\&\EE{I}{\lim_{n\to\infty} \EE{X,Z,\D_n}{
        |f_n(X, Z, \D_n) - f(X)|^2\,|\,I}} = 0
  \end{align*}
  which is the desired result.
\end{proof}

With these preliminary results in hand, we are equipped to prove our main
result.

\begin{theorem}
  Suppose that $X$ is supported on $\R^D$ and has a density which is bounded
  from above and below.  Moreover, suppose that $f(x)$ is bounded and that
  $\E{Y^2} < \infty$.  Then the random regression forest algorithm described in
  this paper is consistent provided that $k_n\to\infty$ and $k_n/n \to 0$ as
  $n\to\infty$.
\end{theorem}
\begin{proof}
  Since the construction of the tree is monotone transformation invariant we can
  assume without loss of generality that $X$ is supported on $[0,1]^D$ with
  uniform marginals \cite{devroye96}.

  By Proposition~\ref{prop:consistent-base} it is sufficient to show consistency
  of the base estimator.  Moreover, using $I$ to denote an infinite sequence of
  partitioning variables, by Proposition~\ref{prop:condition-on-full-measure} it
  is sufficient to show consistency of the base estimator conditioned on $I$.
  To this end, we appeal to Theorem 4.1 from \citet{gyorfi02}.  According to
  this theorem $\{f_n\}$ is consistent if both $\diam(A_n(X)) \to 0$ and
  $N^e(A_n(X))\to\infty$ in probability.  The diameter of a set is defined as
  \begin{align*}
    \diam(A) = \sup_{x, y \in A} \|x-y\|
    \enspace.
  \end{align*}

  Consider a tree partition defined by the structure points (fixed by
  conditioning on $I$) and the additional randomizing variable $Z$.  That
  $N^e(A_n(X))\to\infty$ is trivial, since $N^e(A_n(X)) \ge k_n$.  To see that
  $\diam(A_n(X)) \to 0$ in probability, let $V_n(x)$ be the size of the first
  dimension of $A_n(x)$.  It suffices to show that $\E{V_n(x)} \to 0$ for all
  $x$ in the support of $X$.

  Let $X_1,\ldots,X_{m'} \sim \mu|_{A_n(x)}$ for some $1\le m' \le m$ denote the
  structure points selected to determine the range of the split points in the
  cell $A_n(x)$.  Without loss of generality, we can assume that $V_n(x) = 1$
  and that $\pi_1 X_i \sim \operatorname{Uniform}[0,1]$, where $\pi_1$ is a
  projection onto the first coordinate.  Conditioned on the event that the first
  dimension is cut, the largest possible size for the first dimension of the
  child cells is bounded by
  \begin{align*}
    V^* = \max(\max_{i=1}^m \pi_1X_i, 1-\min_{i=1}^m \pi_1X_i)
  \end{align*}
  Recall that we choose $\min(1+\operatorname{Poisson}(\lambda), D)$ distinct
  candidate split dimensions, and define the following events
  \begin{align*}
    E_1 &= \{\text{There is exactly one candidate dimension}\}
    \\
    E_2 &= \{\text{The first dimension is a candidate}\}
  \end{align*}
  Then, using $V'$ to denote the size of the first dimension of the child cell,
  \begin{align*}
    \E{V'} &\le \E{\1{(E_1\cap E_2)^c} + \1{E_1\cap E_2}V^*}
%     \\
%     &= \P{E_1^c\cup E_2^c} + \P{E_1\cap E_2}\E{V^*}
%     \\
%     &= \P{E_1^c} + \P{E_1 \cup E_2^c} + \P{E_2|E_1}\P{E_1}\E{V^*}
    \\
    &= \P{E_1^c} + \P{E_2^c|E_1}\P{E_1} \\&\phantom{=}\hspace{0.5cm}+
    \P{E_2|E_1}\P{E_1}\E{V^*}
    \\
    &= (1-e^{-\lambda}) + (1-\frac{1}{D})e^{-\lambda} +
    \frac{1}{D}e^{-\lambda}\E{V^*}
%     \\
%     &= 1 - \frac{e^{-\lambda}}{D} + \frac{e^{-\lambda}}{D}\E{V^*}
    %
    \intertext{By Lemma~\ref{lemma:Evstar} in
      Appendix~\ref{appendix:technical},}
    &= 1 - \frac{e^{-\lambda}}{D} +
    \frac{e^{-\lambda}}{D}\cdot \frac{2m+1}{2m+2}
    \\
    &= 1 - \frac{e^{-\lambda}}{2D(m+1)}
  \end{align*}
  Iterating this argument we have that after $K$ splits the expected size of the
  first dimension of the cell containing $x$ is upper bounded by
  \begin{align*}
    \left(1 - \frac{e^{-\lambda}}{2D(m+1)}\right)^K
    \enspace, 
  \end{align*}
  so it suffices to have $K\to\infty$ in probability.  This is shown to hold by
  Proposition~\ref{prop:k-to-inf} in Appendix~\ref{appendix:technical}, which
  proves the claim.
\end{proof}

%%% Local Variables: 
%%% mode: latex
%%% TeX-master: "regression-forests-2013"
%%% End: 

\begin{figure*}[t]
\centering
 {\includegraphics[width=0.24\linewidth]{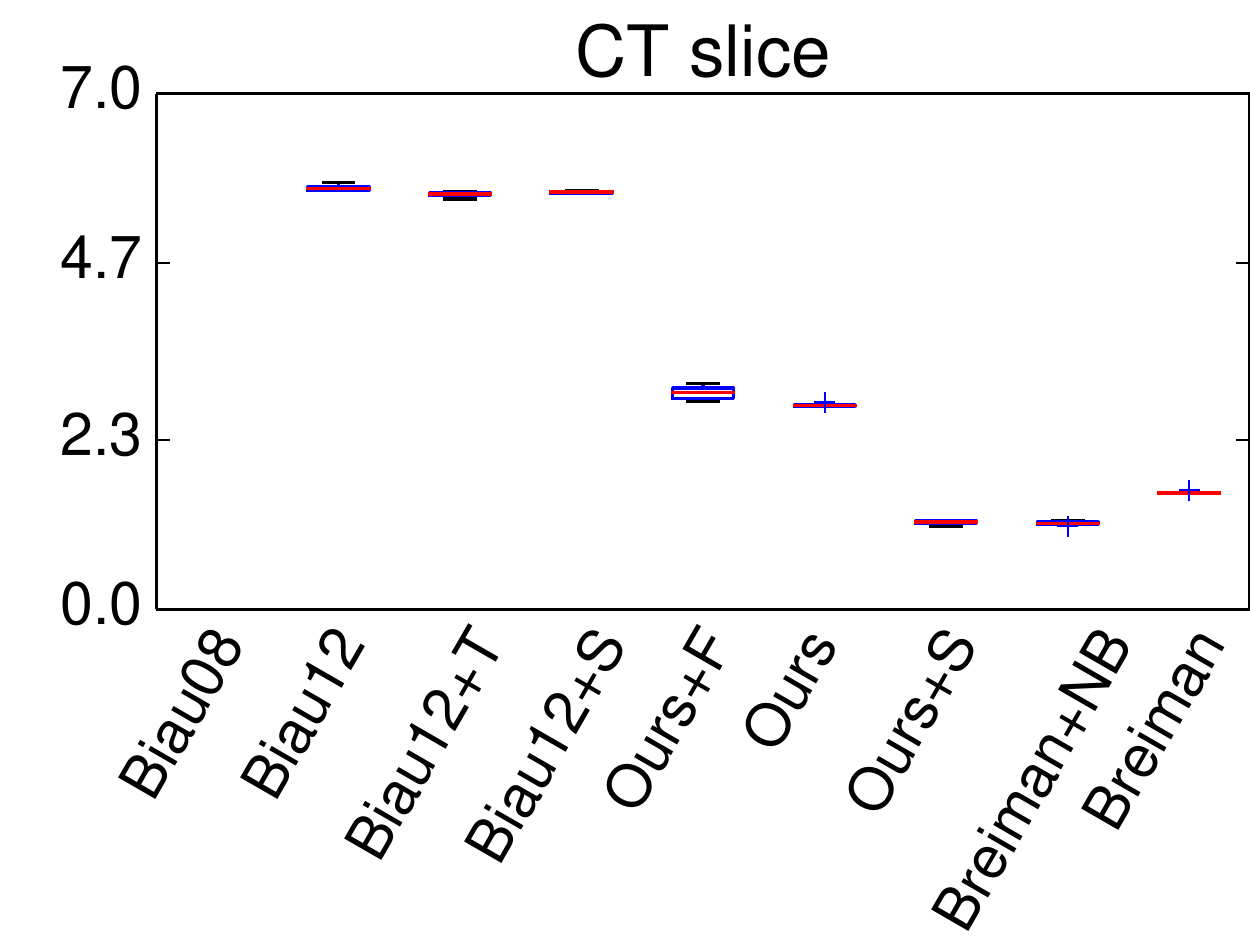}}
 {\includegraphics[width=0.24\linewidth]{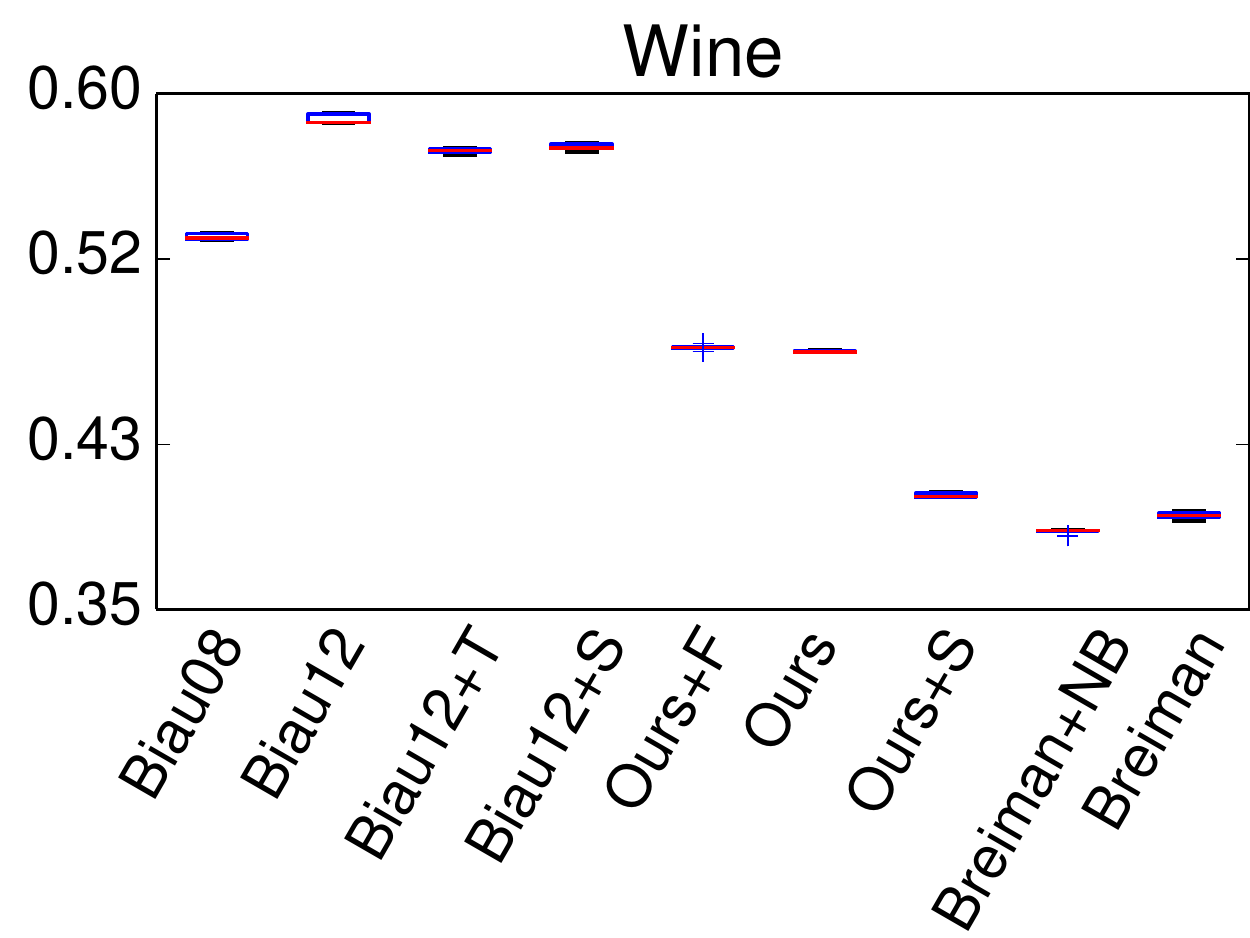}}
 {\includegraphics[width=0.24\linewidth]{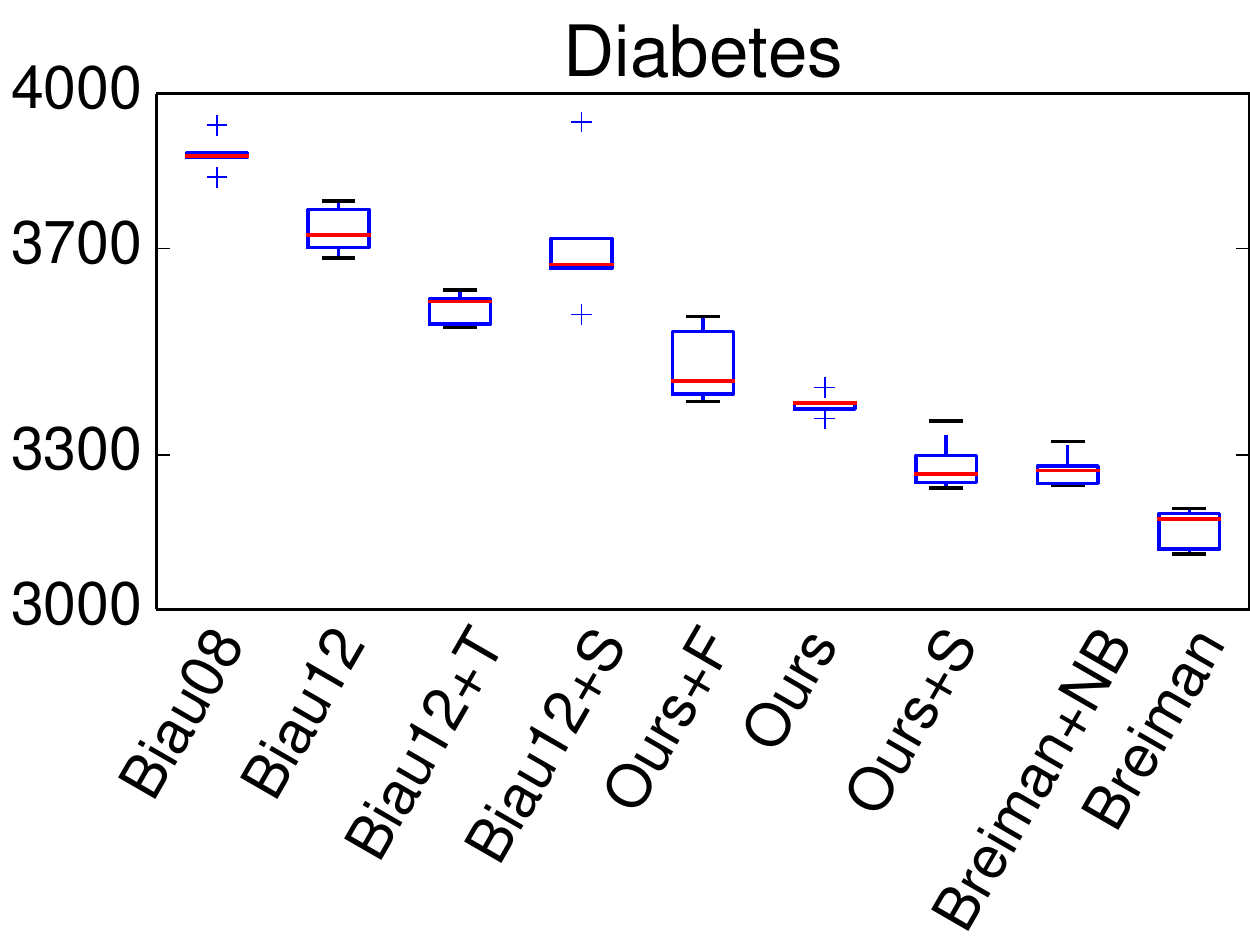}}
 {\includegraphics[width=0.24\linewidth]{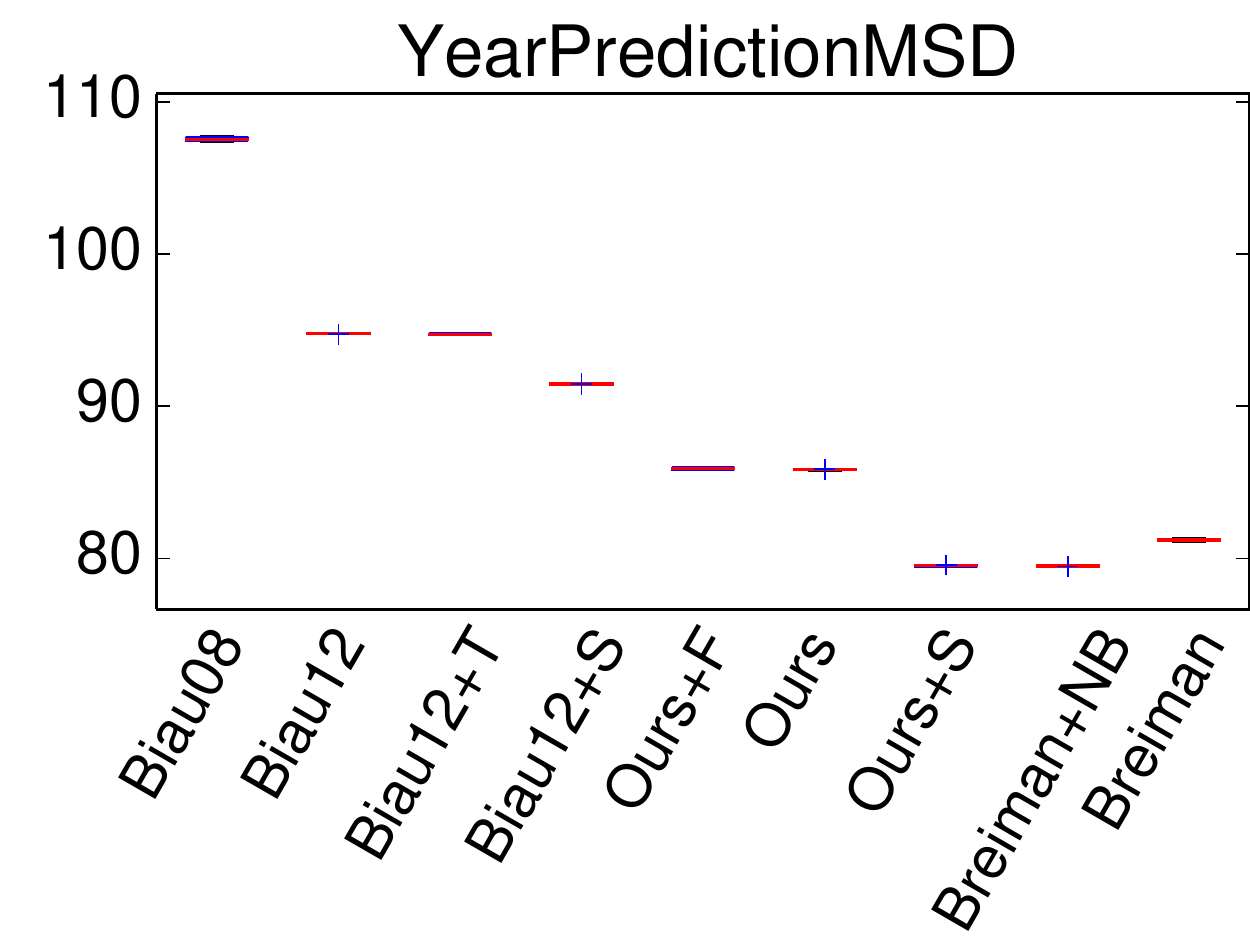}}
 \caption{Comparison between different algorithm permutations on several data
   sets.  In each plot the $y$-axis shows mean squared error, and different
   algorithms are shown along the $x$-axis.  The algorithm in this paper is
   labeled Ours.  Biau08 and Biau12 are algorithms from the literature, and are
   described in the main text.  Breiman is the original random forest algorithm.
   A + sign is used to indicate variants of an algorithm.  +T and +F indicate
   that data splitting is performed at the tree or forest level respectively,
   and +S indicates that no data splitting is used.  Breiman+NB is the original
   random forest algorithm with no bootstrapping.  In the CT slice figure the
   error of Biau08 is not shown, since it is extremely large.}
 \label{fig:uciboxplots}
\end{figure*}

\section{Discussion}
\label{sec:discussion}

In this section we describe two different random forest models which have been
previous analyzed in the literature.  We discuss some of the differences between
them and the model in this paper, and the relationship of the three models to
Breiman's original algorithm.  Both of the models we discuss here were
originally presented as classification algorithms, but adapting them for
regression is straightforward.

The first model we compare to our own is the scale invariant random forest from
\citet{biau08}, which we refer to as Biau08.  The trees in this forest are
constructed by repeatedly expanding leaf nodes as follows: a leaf in the tree is
chosen uniformly at random for expansion.  Within this leaf a dimension is
chosen uniformly at random and the data are sorted according to their projection
into the chosen dimension.  Finally, if there are $N$ data points in the leaf
being expanded then a random index $I$ is drawn from the set $\{0, 1, \ldots,
N\}$ and the split point is chosen so that the $I$ smallest values fall into one
of the children and the rest in the other.  Leaf expansion continues in this
manner until a specified number of terminal nodes has been reached.

The second model we compare to is the algorithm analyzed in \citet{Biau2012},
which we refer to as Biau12.  The trees in this forest assume the data is
supported on $[0,1]^D$, so data must first be scaled to lie in this range.
Trees are grown by expanding leafs in breadth first order until a specified
number of terminal nodes has been reached.  Leafs in this model are expanded by
selecting a fixed number of random candidate dimensions (with replacement).  For
each candidate dimension there is one candidate split point which lies at the
midpoint of the cell being expanded.  To choose between the different candidate
dimensions, the information gain from each split is computed and the candidate
split point with the greatest information gain is selected.

An important feature of Biau12 is that fitting the model requires partitioning
the data set into two parts.  One of these parts is used for determining the
structure of the trees, and the other part is used for fitting the estimators in
the leafs.  The roles of the two parts of this partition are identical to the
structure and estimation points in our own algorithm.  The main difference
between how Biau12 partitions the data and how we do so is that for Biau12
the partition into structure and estimation points is the same for all the trees
in the forest, whereas in our algorithm the partition is randomly chosen
independently for each tree.

Comparing our algorithm and the two from Biau to Breiman's original random
forests algorithm we see there are two key points of difference:
\begin{enumerate}
\item How candidate split points are chosen.
\item How data splitting happens (if at all).
\end{enumerate}

In our experiments we look at how different choices for these two factors effect
the performance of random forests on several regression problems.

% \begin{figure*}[b]
% \centering
%   {\includegraphics[height=0.08\textheight]{figures/kinect-dataset-0}}
%   {\includegraphics[height=0.08\textheight]{figures/kinect-dataset-1}}
%   {\includegraphics[height=0.08\textheight]{figures/kinect-dataset-2}}
%   {\includegraphics[height=0.08\textheight]{figures/kinect-dataset-3}}
%   {\includegraphics[height=0.08\textheight]{figures/kinect-dataset-4}}
%   {\includegraphics[height=0.08\textheight]{figures/kinect-dataset-5}}
%   {\includegraphics[height=0.08\textheight]{figures/kinect-dataset-6}}
%   {\includegraphics[height=0.08\textheight]{figures/kinect-dataset-7}}
%   {\includegraphics[height=0.08\textheight]{figures/kinect-dataset-8}}
%   {\includegraphics[height=0.08\textheight]{figures/kinect-dataset-9}}
%   {\includegraphics[height=0.08\textheight]{figures/kinect-dataset-10}}
%   {\includegraphics[height=0.08\textheight]{figures/kinect-dataset-11}}
%   \caption{Random poses from dataset}
%   \label{fig:kinect-poses}
% \end{figure*}

%%% Local Variables:
%%% mode: latex
%%% TeX-master: "regression-forests-2013"
%%% End:

\section{Experiments}

\begin{table}[b]
{\footnotesize
\begin{tabular}{p{3.2cm}p{1.9cm}p{1.9cm}}
    \hline
    Name & No. data & No. features \\ \hline
    Diabetes & 442 & 10 \\
    Wine Quality & 6497 & 11 \\ 
    YearPredictionMSD & 515345 & 90 \\  
    CT slice & 53500 & 384 \\ \hline
\end{tabular}
}
    \caption{Summary of UCI datasets.}
    \label{table:uci-datasets}
\end{table}

In this section we empirically compare our algorithm to Biau08 and Biau12
(described in Section~\ref{sec:discussion}) and Breiman (the original algorithm
described in \citet{Breiman2001:random_forests}) on several datasets.
%
% The
% purpose is not to improve on state of the art performance; indeed, doing so
% would be quite surprising since we are comparing the basic random forest
% algorithm from 2001 to several simplifications thereof.
%
The purpose of these experiments is to provide insight into the relative impact
of the different simplifications that have been used to obtain theoretical
tractability.  To this end we have chosen to evaluate the different algorithms
on several realistic tasks, including and extremely challenging joint prediction
problem from computer vision.

% In addition to the algorithms mentioned above, we also examine several
% permutations obtained by mixing and matching between candidate split point
% selection and data splitting strategies from the different algorithms.  This
% mix-and-match approach allows us to compare individual features from the
% different algorithms in isolation.

Since the algorithms are each parameterized slightly differently it is not
possible to use the same parameters for all of them.  Breiman and our own
algorithm specify a minimum leaf size, which we set to 5 following Breiman's
advice for regression \cite{Breiman2001:random_forests}.

Biau08 and Biau12 are parameterized in terms of a target number of leafs rather
than a minimum leaf size.  For these algorithms we choose the target number of
leafs to be $n/5$, meaning the trees will be approximately the same size as
those grown by Breiman and our own algorithm.

Biau12 requires the data to lie within the unit hypercube.  For this algorithm
we pre-process the data by shifting and scaling each feature into this range.

\subsection{UCI datasets}

\begin{figure}[tb]
\centering
  {\includegraphics[width=0.39\linewidth]{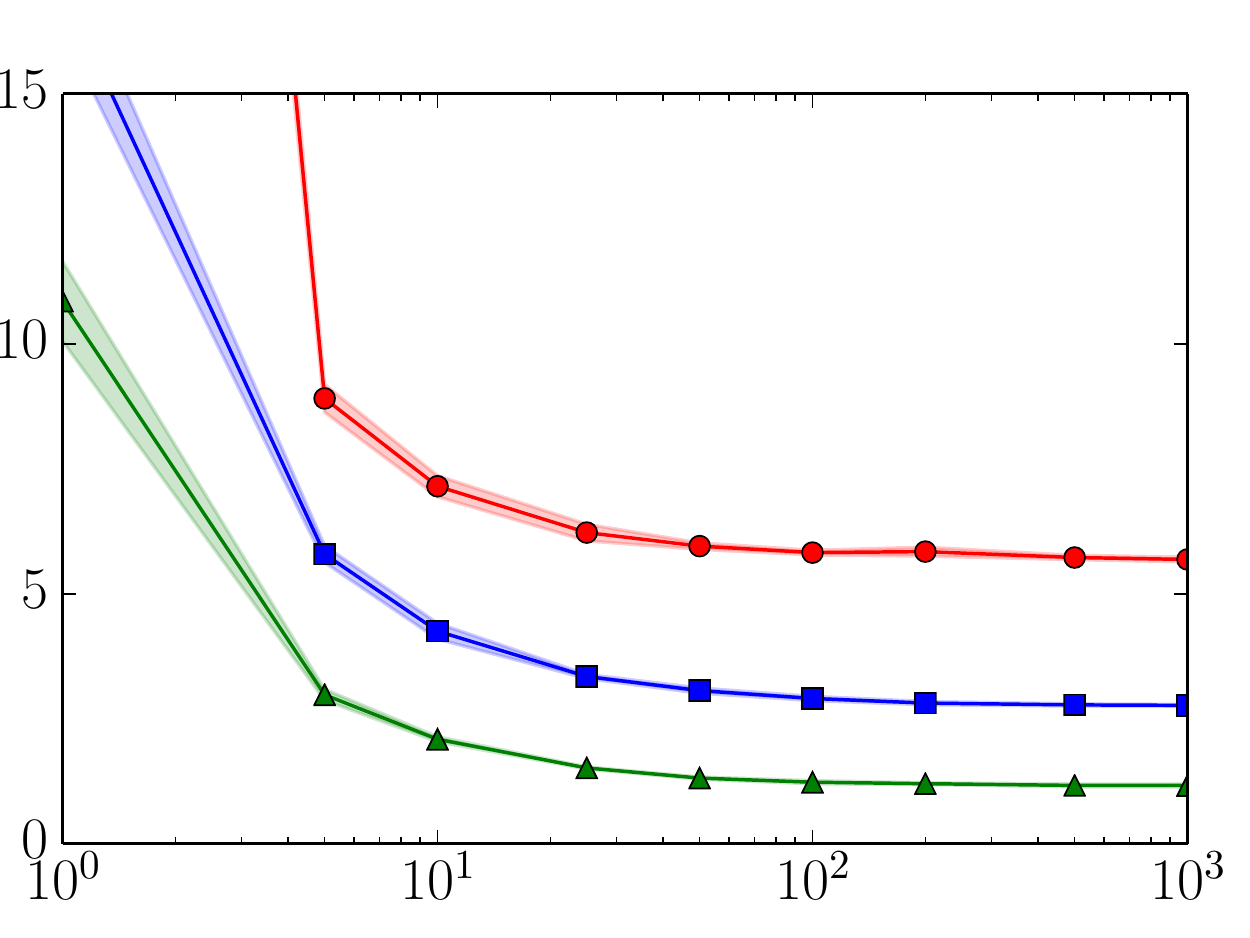}}
  {\includegraphics[width=0.59\linewidth]{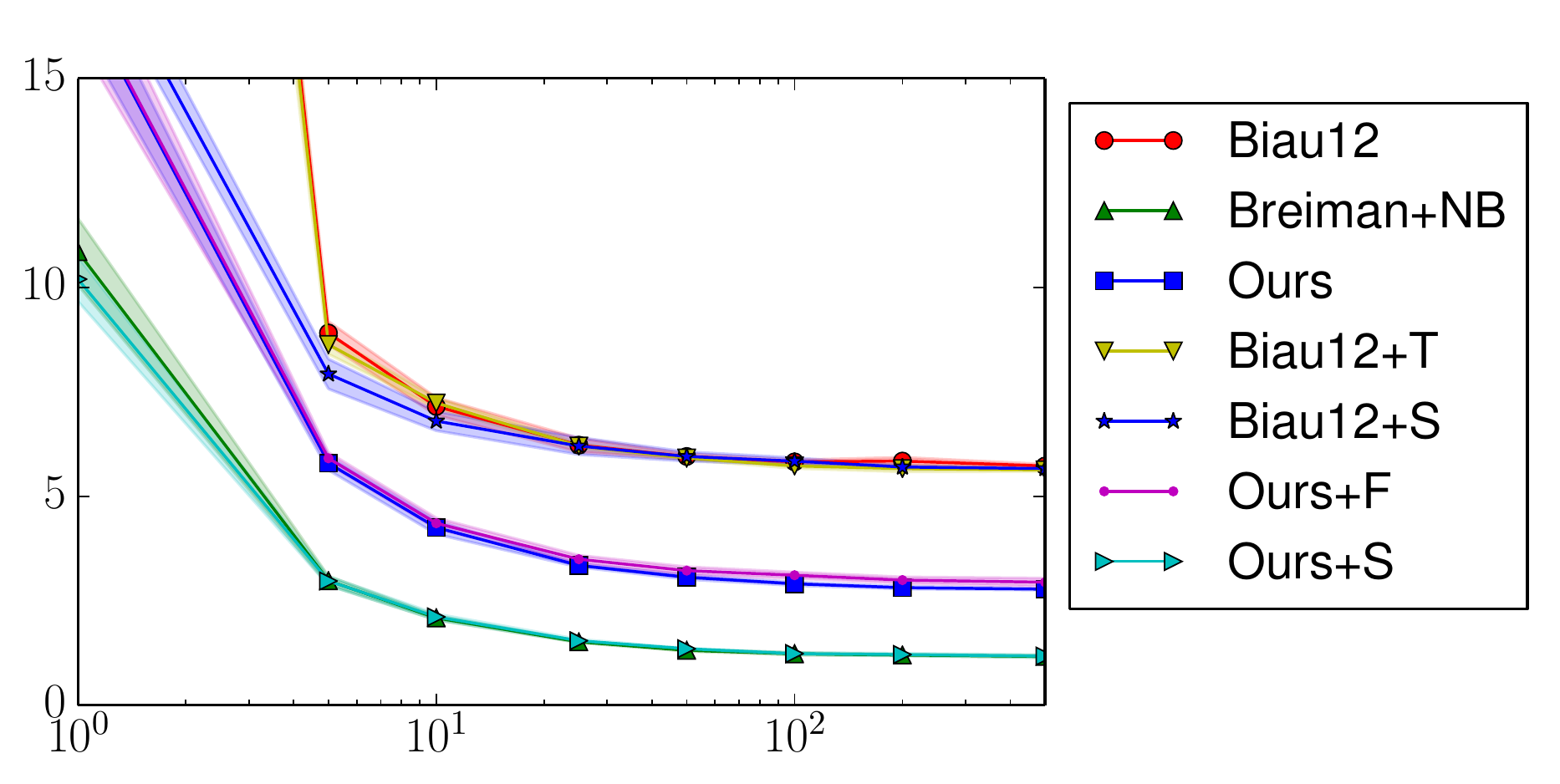}}
  \caption{\textbf{Left:} Performance comparison as a function of forest size
    for the different algorithms on the CT slice data set.  \textbf{Right:}
    Comparison between different methods of data splitting and split point
    selection on the CT slice dataset.  In both plots the $x$-axis number of
    trees and the $y$-axis is mean squared error.  Error bars show one standard
    deviation computed over five runs.  Biau08 does not appear in either plot
    since its error in this dataset is very large.  See the caption of
    Figure~\ref{fig:uciboxplots} for an explanation of the labels.}
  \label{fig:uciextra}
\end{figure}

For our first set of experiments we used four data sets from the UCI repository.
A summary of the datasets can be seen in Table~\ref{table:uci-datasets}.  With
the exception of diabetes, these datasets were chosen for their relatively large
number of instances and features.

In all the experiments in this section we follow Breiman's rule of thumb of
using one third of the total number of attributes as candidate dimensions.  All
results in the this section are the mean of five runs of five fold cross
validation.

For our algorithm we choose $m=1000$ structure points for selecting the search
range in the candidate dimensions.  We experimented with other settings for $m$
but found our results to be very insensitive to this parameter.

Figure~\ref{fig:uciboxplots} compares the performance of several different
random forest algorithm variants on the four UCI data sets.  The clear trend
here is that Breiman's algorithm outperforms our own, which in turn outperforms
both algorithms from Biau.  Generally Biau12 outperforms Biau08, except in the
wine quality data set where, strangely, the order is reversed. 

Figure~\ref{fig:uciboxplots} includes a variant of our algorithm which performs
data splitting at the forest level, and also a variant of Biau12 which performs
data splitting at the tree level.  This difference appears to have relatively
little effect when there is sufficient data; however, for the Diabetes dataset,
which is comparatively small, splitting at the tree instead of the forest level
significantly improves performance.

In all cases the gap between Biau12 and our algorithm is larger than the
difference in performance from changing how data splitting is done.  This
indicates that in a practical sense it is the split selection strategy that
accounts for most of the improvement of our algorithm over Biau12.

We also experimented with variants of Biau12 and our own algorithm with no data
splitting.  The most notable thing here is that when data splitting is removed
our algorithm is very competitive with Breiman.  This indicates that the gap in
performance between our algorithm and standard random forests can be contributed
almost entirely to data splitting.

We performed all of these experiments using a range of forest sizes.
Figure~\ref{fig:uciextra} (left) shows performance as a function of forest size.
In the interest of space we present this figure only for the CT slice dataset,
but the curves for the other datasets tell a similar story.  This figure shows
that the results from Figure~\ref{fig:uciboxplots} are consistent over a wide
range of forest sizes.

Figure~\ref{fig:uciextra} (right) more closely examines the effects of the
different data splitting and split point selection strategies.

\subsection{Kinect Pose Estimation}

\begin{figure}[t]
\centering
  {\includegraphics[width=0.2\linewidth]{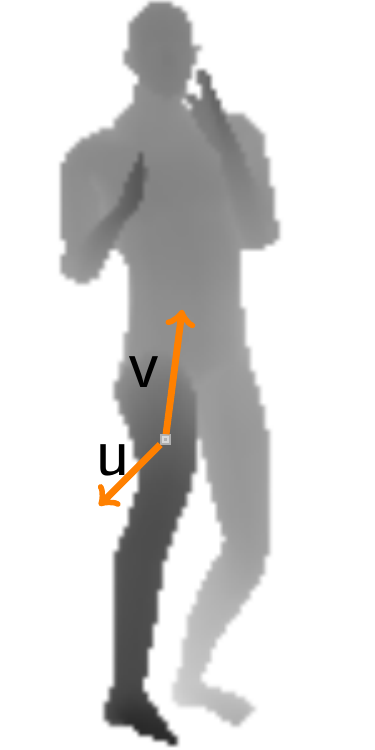}}
  \quad
  {\includegraphics[width=0.2\linewidth]{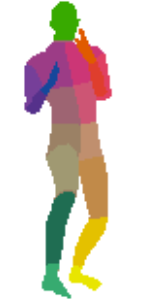}}
  \quad
  {\includegraphics[width=0.2\linewidth]{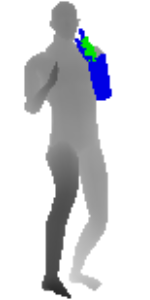}}
  \caption{\textbf{Left:} Depth image with a candidate feature specified by the
    offsets u and v.  \textbf{Center:} Body part labels.  \textbf{Right:} Left hand
    joint predictions (green) made by the appropriate class pixels (blue).}
  \label{fig:kinect-example}
\end{figure}

In this section, we evaluate our random forest algorithm on the challenging
computer vision problem of predicting the location of human joints from a depth
image and corresponding body part labels.  See Figure~\ref{fig:kinect-example}
for an example.

Typically the first step in a joint location pipeline is to predict the body part
labels of each pixel in the depth image and the second step is to use the
labelled pixels to predict joint locations \cite{Shotton2011}.
% Further
% refinements to this procedure can predict both the pixel label and joint
% locations simultaneously using a Hough forest as in \citet{Girshick2011};
% however these refinements are well beyond the scope of this paper.

Since our primary goal is to evaluate regression models rather than to build an
end product, we implement only the second step in the basic pipeline.  Using
depth images with ground truth body part labels for each pixel as training data,
we learn a regression model of the offset from a pixel to a joint.

For each joint, we train a forest on the pixels of body parts associated
with that joint and predict the relative offset from each pixel to the joint.
Typically these errors would be post-processed with mean shift to find a more
accurate final prediction for the joint location.  We instead report the
regression error directly to avoid confounding factors in the comparison between the
forest models.

Each joint has its own model that predicts the offset from a pixel to the
location of the joint.  An offset is predicted for all pixels with body part
labels associated with a joint.

\begin{figure}[t]
  \centering
  \includegraphics[width=0.9\linewidth]{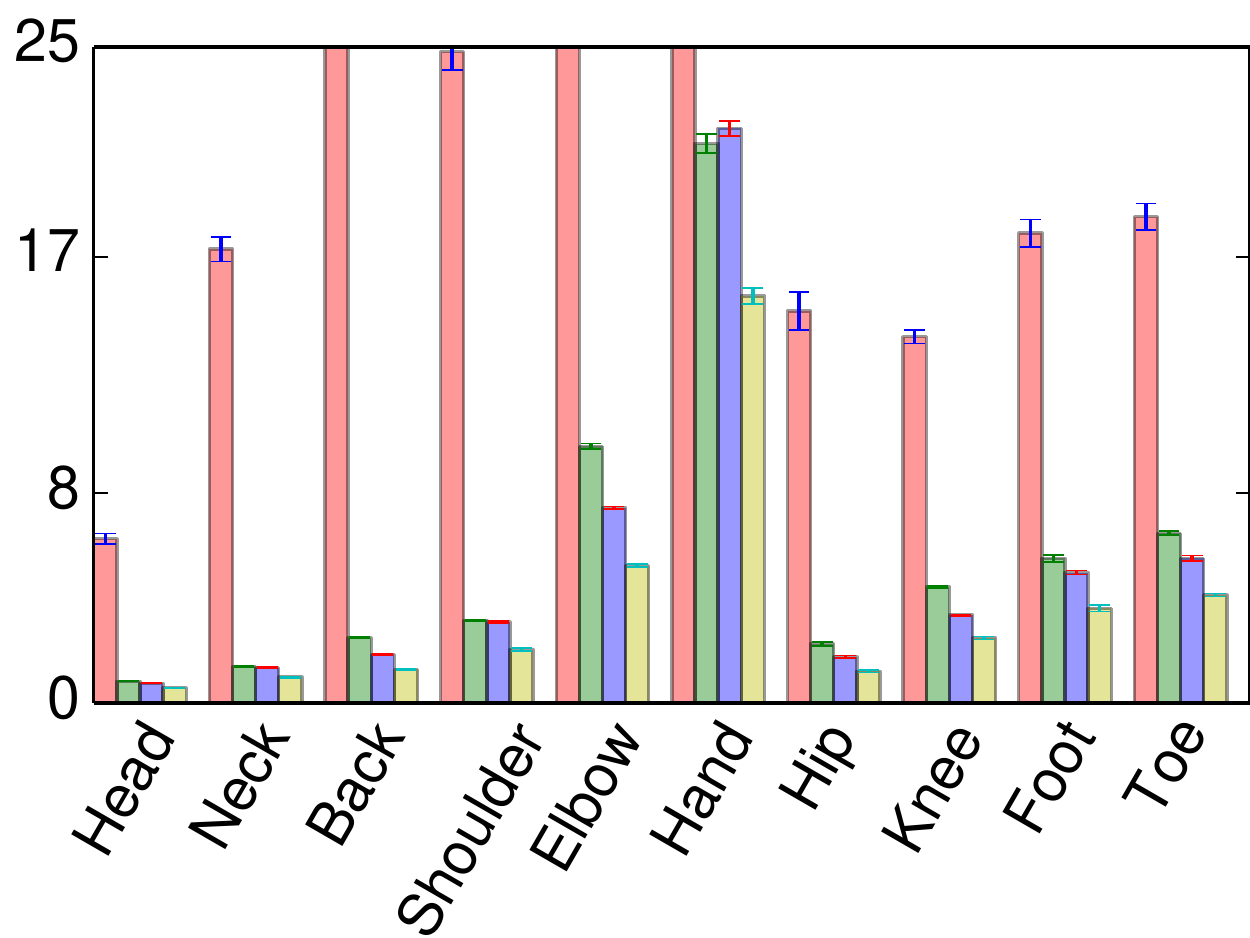}
  \caption{Mean squared error in pixel space for a forest of 50 trees on the
    kinect joint prediction task.  Each group of bars shows, from left to right,
    the error of Biau08, Biau12, Ours and Breiman.  The error bars show one
    standard deviation across 5 runs.  Due to space we only include the errors
    the left side of the body but the results for the right side are similar.
    In order to make the results legible the $y$-axis is set so that in some
    cases the error of Biau08 extends vertically off the figure.}
  \label{fig:kinect-error}
\end{figure}

To build our data set, we sample random poses from the CMU mocap
  dataset\footnote{Data obtained from \url{mocap.cs.cmu.edu}} and render a pair
of 320x240 resolution depth and body part images along with the positions of
each joint in the skeleton.  The 19 body parts and one background class are
represented by 20 unique color identifiers in the body part
image. 

For this experiment we generate 2000 poses for training and 500 poses for
testing.  To create the training set, we sample 20 pixels without replacement
from each body part class in each pose.  We then sample 40000 pixels without
replacement from the sampled pixels with the associated body part labels across
all poses.  During testing we evaluate the MSE of the offsets of all pixels
associated with a joint.  Figure~\ref{fig:kinect-example} visualizes the raw
depth image, ground truth body part labels and the votes for the left hand made
by all pixels in the left arm.

The features associated with each pixel are depth differences between pairs of
pixels at specified offsets from the target.  At training time, candidate pairs
of offsets are sampled from a 2-dimensional Gaussian distributions with variance
40.0 (chosen by cross validation). The offsets are scaled by the depth of the
target pixel to produce depth invariant
features. Figure~\ref{fig:kinect-example}~(left) shows candidate feature offsets
$u$ and $v$ for the indicated pixel. The resulting feature value is the depth
difference between the pixel at offset $u$ and the pixel at offset $v$.  In this
experiment we sample 1000 candidate offsets at each node.

Figure~\ref{fig:kinect-error} shows the MSE and standard deviation for each
joint in pixel units.  In the interest of space we only show the joints for the
left side of the body but we see similar results for the right side.  Just as
with the UCI datasets, the dominant ordering from largest to smallest test error
is Biau08, Biau12, ours and Breiman.

\section{Conclusion}

It is fascinating that an algorithm as simple and useful as random forests has
turned out to be so difficult to analyze. Motivated by this, we set as our goal
to narrow the gap between the theory and practice of regression forests, and we
succeeded to a significant extent. Specifically, we were able to derive a new
regression forest algorithm, to prove that it is consistent, and to show that
its empirical performance is closer to Breiman's popular model than previous
theoretical variants.

Our extensive empirical study, which compares the algorithm widely used in
practice to recent theoretical variants for the first time, also casts light on
how different design choices and theoretical simplifications impact performance.

We focused on consistency because this is still an important open
problem. However, we believe that our theoretical analysis and empirical study
help in setting the arena for embarking on other types of analyses, including
finite sample size complexity bounds, asymptotic convergence rates, and
consistency of random forests in machine learning problems beyond regression.

%%% Local Variables:
%%% mode: latex
%%% TeX-master: "regression-forests-2013"
%%% End:

%\vspace{-0.1cm}
%\section*{Acknowledgements}
%\vspace{-0.1cm}

%Some of the data used in this paper was obtained from \url{mocap.cs.cmu.edu}
%(funded by NSF EIA-0196217).

\clearpage
\small{
\subsubsection*{References}
\bibliography{randomforests}
\bibliographystyle{icml2014}
}

\appendix
\clearpage

\onecolumn
\section{Technical results}
\label{appendix:technical}

\begin{lemma}
  If $U_1, \ldots U_m$ are iid $\operatorname{Uniform}[0,1]$ random variables
  then
  \begin{align*}
    \E{\max(\max_{i=1}^m U_i, 1-\min_{i=1}^m U_i)} = \frac{2m+1}{2m+2}
  \end{align*}
  \label{lemma:Evstar}
\end{lemma}
\begin{proof}
  Let $M_i = \max(U_i, 1-U_i)$, so $M_i$ are iid $\operatorname{Uniform}[1/2,
  1]$ with CDF given by
  \begin{align*}
    F_{M_i}(x) = 2x-1
  \end{align*}
  for $1/2 \le x \le 1$.  Moreover, if $M = \max_{i=1}^m M_i$ then $F_M(x) =
  (2x-1)^m$ since the $M_i$ are iid.  The density of $M$ is then
  \begin{align*}
    f_M(x) = \frac{\mathrm{d}}{\mathrm{d}x}F_M(x) = 2m(2x-1)^{m-1}
  \end{align*}
  and its expected value is
  \begin{align*}
    \E{M} = \int_{1/2}^1 x f_M(dx) = \frac{2m+1}{2m+2}
  \end{align*}
  which proves the claim.
\end{proof}

% \begin{lemma}
%   Suppose $\{U_1\ldots,U_m\}$ are uniform iid on $\{1, \ldots, n\}$ and let $X =
%   \max_{i=1}^m \{U_i, n+1-U_i\}$.   Then for $k \le \floor{n/2}$ we have
%   \begin{align*}
%     \P{X \le n-k} \ge \left(1-\frac{2k}{n}\right)^m
%   \end{align*}
%   In particular, $\P{X\le n-k} \to 1$ as $n\to\infty$.
%   \label{lemma:discrete-uniform}
% \end{lemma}
% \begin{proof}
%   Let $X_i = \max\{U_i, n+1-U_i\}$, which is uniform on $\{\ceil{n/2}, \ldots,
%     n\}$, so for $0 \le k \le \floor{n/2}$
%     \begin{align*}
%       \P{X_i \le n-k} &= \frac{\floor{n-k}-\ceil{n/2}+1}{n-\ceil{n/2}+1} \ge
%       %1-\frac{2k}{n+2} \ge
%       1-\frac{2k}{n}
%       % \\
%       % &= \frac{\floor{n/2-k+1}}{\floor{n/2+1}}
%       % \\
%       % &\ge \frac{n/2 - k}{n/2 + 1}
%       % \\
%       % &= \frac{n - 2k}{n + 2}
%       % \\
%       % &= 1-\frac{2k}{n+2}
%     \end{align*}
%     Moreover, since the $X_i$ are iid we have
%     \begin{align*}
%       \P{X \le n-k} = \prod_{i=1}^m\P{X_i \le n-k} \ge
%       \left(1-\frac{2k}{n}\right)^m
%     \end{align*}
%     and the claim is shown.
% \end{proof}

\begin{proposition}
  For sufficiently large $n$, every cell of the tree will be cut infinitely
  often in probability.  That is, if $K$ is the distance from the root of the
  tree to a leaf then $P(K<t) \to 0$ for all $t$ as $n\to\infty$.
  \label{prop:k-to-inf}
\end{proposition}

\begin{proof}
  The splitting mechanism functions by choosing $m$ structure points uniformly
  at random from the node to be split and searching between their min and max.
  We will refer to the points selected by the splitting mechanism as active.
  Without loss of generality we can assume the active points are uniformly
  distributed on $[0, 1]$ and lower bound the number of estimation points in the
  smallest child.

  Denote the active points $U_1, \ldots, U_m$ and let $U = \max_{i=1}^m(\max(U_i,
  1-U_i))$.  We know from the calculations in Lemma~\ref{lemma:Evstar} that
  \begin{align*}
    \P{U \le t} = (2t-1)^m
  \end{align*}
  which means that the length of the smallest child is at least $\delta^{1/K} < 1$
  with probability $(2(1-\delta^{1/K})-1)^m$, i.e.\
  \begin{align*}
    \P{U \le 1-\delta^{1/K}} = (2(1-\delta^{1/K})-1)^m
  \end{align*}
  Repeating this argument $K$ times we have that after $K$ splits all sides of
  all children have length at least $\delta$ with probability at least
  $(2(1-\delta^{1/K})-1)^{Km}$.  This bound is derived by assuming that the same
  dimension is cut at each level of the tree. If different dimensions are cut at
  different levels the probability that all sides have length at least $\delta$
  is greater, so the bound holds in those cases also.

  This argument shows that every cell at depth $K$ contains a hypercube with
  sides of length $\delta$ with probability at least
  $(2(1-\delta^{1/K})-1)^{Km}$.  Thus for any $K$ and $\epsilon_1>0$ we can pick
  $\delta$ such that
  \begin{align*}
    % (2(1-\delta^{1/K})-1)^{Km} & \ge 1-\epsilon_1
    % \\
    % 2(1-\delta^{1/K})-1 &\ge (1-\epsilon_1)^{1/Km}
    % \\
    % 1-\delta^{1/K} &\ge \frac{1}{2}((1-\epsilon_1)^{1/Km}+1)
    % \\
    0 < \delta^{1/K} &\le 1-\frac{1}{2}((1-\epsilon_1)^{1/Km}+1)
  \end{align*}
  and know that every cell of depth $K$ contains a hypercube with sides of
  length $\delta$ with probability at least $1-\epsilon_1$.  Since the
  distribution of $X$ has a non-zero density, each of these hypercubes has
  positive measure with respect to $\mu_X$.  Define
  \begin{align*}
    p = \min_{L \text{ a leaf at depth } K}\mu_X(L)
    \enspace.
  \end{align*}
  We know $p > 0$ since the minimum is over finitely many leafs and each leaf
  contains a set of positive measure.

  It remains to show that we can choose $n$ large enough so that any set $A
  \subset [0, 1]^D$ with $\mu_X(A) \ge p$ contains at least $k_n$ estimation
  points.  To this end, fix an arbitrary $A\subset [0,1]^D$ with $\mu_X(A) = p$.
  In a data set of size $n$ the number of points which fall in $A$ is
  $\operatorname{Binomial}(n, p)$.  Each point is an estimation point with
  probability $1/2$, meaning that the number of estimation points, $E_n$, in $A$
  is $\operatorname{Binomial}(n,p/2)$.

  % To see this, note that
  % \begin{align*}
  %   E_n = \sum_{i=1}^N X_i
  % \end{align*}
  % where $N \sim \operatorname{Binomial}(n, p)$ is the number of points which
  % fall in $A$ and $X_i \sim \operatorname{Bernoulli}(1/2)$ is equal to 1 iff the
  % $i$th data point is an estimation point, and zero otherwise.  Using $G_N(z)$
  % to denote the probability generating function of $N$, and $G_X(z)$ to denote
  % the common probability generating function of the $X_i$s, we can derive the
  % probability generating function for $E_n$ as follows:
  % \begin{align*}
  %   G_{E_n}(z) = G_N(G_X(z)) = ((1-p) + p(1/2 + z/2))^n = ((1-p/2) + pz/2)^n
  % \end{align*}
  % which we recognize as the probability generating function of a
  % $\operatorname{Bernoulli}(n,p/2)$ random variable.

  Using Hoeffding's inequality we can bound $E_n$ as follows
  \begin{align*}
    % \P{E_n < k_n} &< \epsilon
    % \\
    % \E{E_n} &= \frac{np}{2}
    % \\
    % t &= \frac{np}{2} - k_n
    % \\
    % \P{E_n < \E{E_n} - t} &< \exp{-2t^2/n}
    % \\
    % \P{E_n < \frac{np}{2} - \frac{np}{2} + k_n} &< \exp{-\frac{2}{n}
    %   \left(\frac{np}{2} - k_n \right)^2}
    % \\
    \P{E_n < k_n} &\le \exp{-\frac{2}{n} \left(\frac{np}{2} - k_n \right)^2}
    % \\
    % \P{E_n < k_n} &< \exp{-\frac{2}{n} \left(\frac{n^2p^2}{4} - k_n\frac{np}{2}
    %     + k_n^2 \right)}
    % \\
    % &= \exp{-\frac{np^2}{2} + k_np - \frac{2k_n^2}{n}}
    % \\
    % &= \exp{(k_n-\frac{np}{2})p}\exp{ - \frac{2k_n^2}{n}}
    % \\
    \le \exp{(k_n-\frac{np}{2})p}
    \enspace.
  \end{align*}
  For this probability to be upper bounded by an arbitrary $\epsilon_2>0$ it is
  sufficient to have
  \begin{align*}
    % \epsilon_2 &\ge \exp{(k_n-\frac{np}{2})p}
    % \\
    % \log(\epsilon_2) &\ge (k_n-\frac{np}{2})p
    % \\
    % -\log(\frac{1}{\epsilon_2}) &\ge (k_n - \frac{np}{2})p
    % \\
    % -\frac{1}{p}\log(\frac{1}{\epsilon_2}) &\ge k_n - \frac{np}{2}
    % \\
    % -\frac{2}{p}\log(\frac{1}{\epsilon_2}) &\ge 2k_n - np
    % \\
    % np - \frac{2}{p}\log(\frac{1}{\epsilon_2}) &\ge 2k_n
    % \\
    % 2k_n &\le np - \frac{2}{p}\log(\frac{1}{\epsilon_2})
    % \\
    \frac{k_n}{n} &\le \frac{p}{2} - \frac{1}{np}\log(\frac{1}{\epsilon_2})
    \enspace.
  \end{align*}
  The second term goes to zero as $n \to \infty$ so for sufficiently large $n$
  the RHS is positive and since $k_n/n \to 0$ it is always possible to choose
  $n$ to satisfy this inequality.

  In summary, we have shown that if a branch of the tree is grown to depth $K$
  then the leaf at the end of this branch contains a set of positive measure
  with respect to $\mu_X$ with arbitrarily high probability.  Moreover, we have
  shown that if $n$ is sufficiently large this leaf will contain at least $k_n$
  estimation points.

  The only condition which causes our algorithm to terminate leaf expansion is
  if it is not possible to create child leafs with at least $k_n$ points.  Since
  we can make the probability that any leaf at depth $K$ contains at least this
  many points arbitrarily high, we conclude that by making $n$ large we can make
  the probability that all branches are actually grown to depth at least $K$ by
  our algorithm arbitrarily high as well.  Since this argument holds for any $K$
  the claim is shown.
\end{proof}

% \begin{figure}[t]
%   \centering
%   \includegraphics[width=0.5\linewidth]{figures/repeated-splitting}
%   \caption{Illustration of the argument structure in
%     Proposition~\ref{prop:k-to-inf}.  Vertical lines represent the chosen split
%     dimension, with the solid red section of each line showing the range that is
%     searched for split points.  The dashed regions of each line may contain data
%     points but a split cannot occur in these regions.  The numbers above each
%     line are a lower bound on the number of data points which must appear at
%     that level of the tree.  The argument proceeds by choosing an appropriate
%     value of $n_K$ and deriving a lower bound on $n_0$ which causes the chosen
%     value of $n_K$ to be exceeded with high probability.}
%   \label{fig:repeated-splitting}
% \end{figure}

%%% Local Variables: 
%%% mode: latex
%%% TeX-master: "regression-forests-2013"
%%% End: 

\end{document}